\documentclass[letterpaper]{article} 
\usepackage{aaai23}  
\usepackage{times}  
\usepackage{helvet}  
\usepackage{courier}  
\usepackage[hyphens]{url}  
\usepackage{graphicx} 
\usepackage{natbib}  
\usepackage{caption} 
\usepackage{algorithm}
\usepackage{algorithmic}
\usepackage{balance}

\usepackage{makecell}
\usepackage[utf8]{inputenc}
\usepackage{amsmath}
\usepackage{amssymb}
\usepackage{amsthm}
\usepackage{enumitem}

\usepackage{caption} 

\usepackage{comment}
\usepackage{multirow}
\usepackage{longtable,tabularx,booktabs}
\newcommand{\cmark}{\ding{51}}
\newcommand{\xmark}{\ding{55}}
\usepackage{pifont}
\usepackage{amsfonts}
\usepackage{amssymb}
\usepackage{amsmath,amsfonts}
\usepackage{algorithmic}
\usepackage{algorithm}
\usepackage{array}
\usepackage[caption=false,font=normalsize,labelfont=sf,textfont=sf]{subfig}
\usepackage{textcomp}
\usepackage{stfloats}
\usepackage{url}
\usepackage{verbatim}
\usepackage{graphicx}

\usepackage{comment}
\usepackage{amsthm}
\theoremstyle{plain}
\newtheorem{definition}{Definition}
\newtheorem{theorem}{Theorem}
\newtheorem{lemma}{Lemma}

\usepackage{newfloat}
\usepackage{listings}
\DeclareCaptionStyle{ruled}{labelfont=normalfont,labelsep=colon,strut=off} 
\floatstyle{ruled}
\newfloat{listing}{tb}{lst}{}
\floatname{listing}{Listing}
\title{Reachability Analysis of Neural Network Control Systems}

\author {
    Chi Zhang\textsuperscript{\rm 1},
    Wenjie Ruan\textsuperscript{\rm 1}$^{\dagger}$,
    Peipei Xu \textsuperscript{\rm 2}
}
\affiliations {
    \textsuperscript{\rm 1} Department of Computer Science, University of Exeter, Exeter, EX4 4QF, UK\\
    \textsuperscript{\rm 2} Department of Computer Science, University of Liverpool, Liverpool, L69 3BX, UK\\
    \{cz338; w.ruan\}@exeter.ac.uk,~~peipei.xu@liverpool.ac.uk\\
    $\dagger~$Corresponding Author
}

\usepackage{bibentry}

\begin{document}

\maketitle

\begin{abstract}
Neural network controllers (NNCs) have shown great promise in autonomous and cyber-physical systems. Despite the various verification approaches for neural networks, the safety analysis of NNCs remains an open problem. Existing verification approaches for neural network control systems (NNCSs) either can only work on a limited type of activation functions, or result in non-trivial over-approximation errors with time evolving. This paper proposes a verification framework for NNCS based on Lipschitzian optimisation, called DeepNNC. We first prove the Lipschitz continuity of closed-loop NNCSs by unrolling and eliminating the loops. We then reveal the working principles of applying Lipschitzian optimisation on NNCS verification and illustrate it by verifying an adaptive cruise control model. Compared to state-of-the-art verification approaches, DeepNNC shows superior performance in terms of efficiency and accuracy over a wide range of NNCs. We also provide a case study to demonstrate the capability of DeepNNC to handle a real-world, practical, and complex system. Our tool \textbf{DeepNNC} is available at \url{https://github.com/TrustAI/DeepNNC}.
\end{abstract}

\section{Introduction}

Neural network controllers have gained increasing interest in the autonomous industry and cyber-physical systems because of their excellent capacity of representation learning \cite{ding2019survey,huang2020survey,wang2022deep}. Compared to conventional knowledge-based controllers, learning-based NNCs not only simplify the design process, but also show superior performance in various unexpected scenarios~\cite{schwarting2018planning,kuutti2020survey}. Since they have already been applied in safety-critical circumstances, including autonomous driving~\cite{bojarski2016end,wu2021adversarial,yin2022dimba} and air traffic collision avoidance systems~\cite{julian2016policy}, verification on NNCSs plays a crucial role in ensuring their safety and reliability before their deployment in the real world. Essentially, we can achieve verification on NNCSs through estimating the system's reachable set. As Figure \ref{flow} shows, given the neural network controller, the dynamic model of the target plant, and the range $X_0$ of initial states, we intend to estimate the reachable sets of the system, that is, the output range of state variables over a finite time horizon $[0, t_n]$. If the reachable sets $X_{t1}$, $X_{t2}$, ... $X_{tn}$ have {\em no intersection} with the avoid set and $X_{tn}$ reaches the goal set, we can verify that the NNCS is safe.

\begin{figure}[t]
\begin{center}
\centerline{\includegraphics[width=0.65\columnwidth]{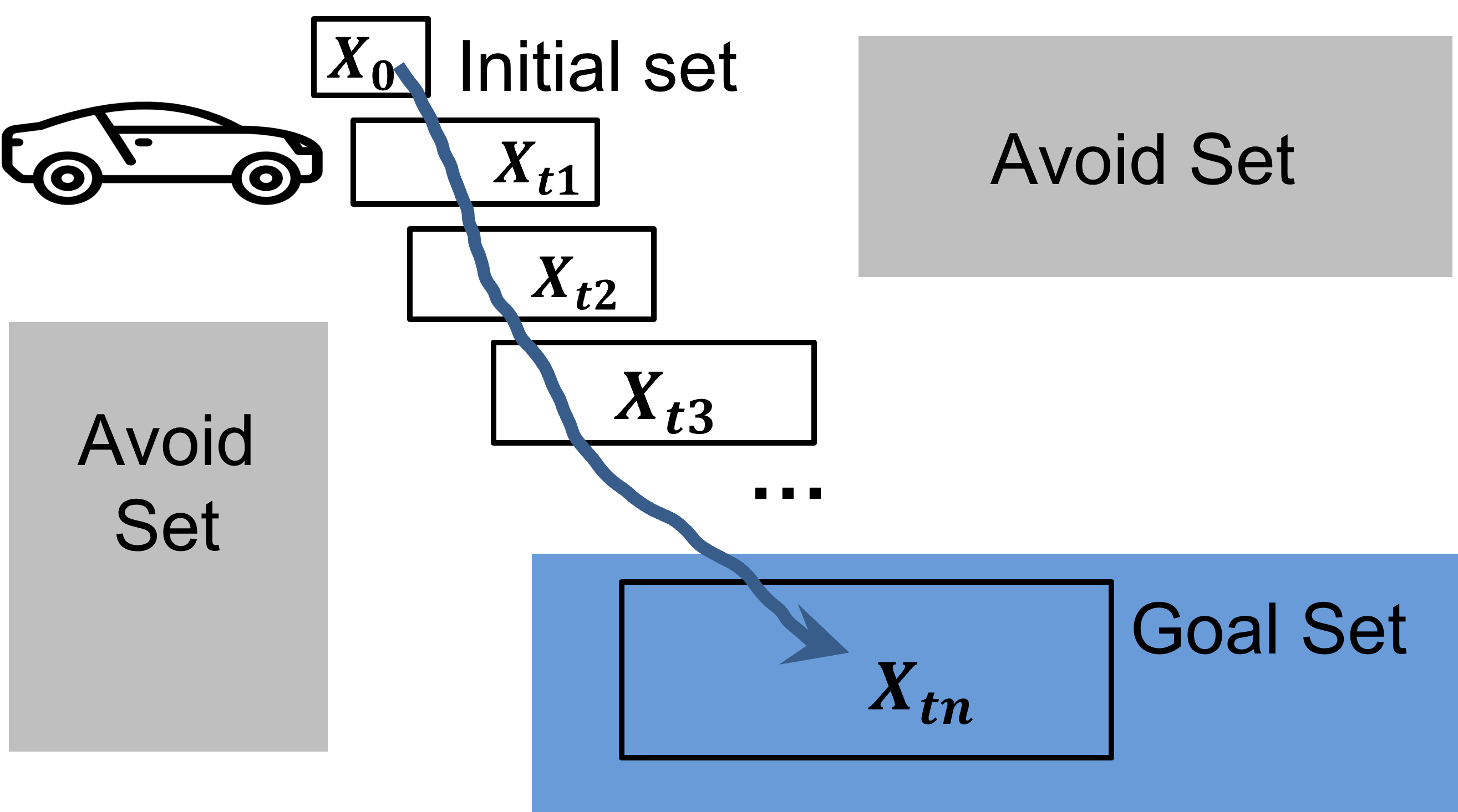}}
\caption{NNCS verification through reachable set estimation. {\em We estimate the reachable sets $X_{t1}, X_{t2}, X_{t3}$ ..., $X_{tn}$ at $t_1,t_2,t_3$ ..., $t_n$. If all reachable sets do not intersect with the grey avoid set and $X_{tn}$ reaches the blue goal set, the NNCS is verified to be safe.}}
\label{flow}
\end{center}
\vspace{-8mm}
\end{figure}


Compared to verification on NNCSs, verification on neural networks (NNs) is a relatively well-explored area~\cite{ruan2018reachability,ruan2019global,wu2020game,ruan2021adversarial,zhang2022proa}. Representative solutions can be categorised into constraint satisfaction based approaches~\cite{botoeva2020efficient, katz2019marabou, katz2017reluplex} and approximation-based approaches~\cite{singh2019abstract, gehr2018ai2, ryou2021scalable,mu20223dverifier}.
However, these verification approaches cannot be applied directly to NNCS. They are not workable in the scenario where the neural network is in close conjunction with other subsystems. A speciality of NNCSs is the association with control time steps. Moreover, approximation-based verification inevitably results in the accumulation of the over-approximation error with the control time evolved.

\begin{table*}[t]
  \centering
  \caption{Comparison with existing NNCS verification methods from multiple aspects}
  \label{tab-comparison}
  \scalebox{0.75}{
  \begin{tabular}{l|ccc ccc ccc}
     \toprule
     & \multicolumn{1}{|c}{\vtop{\hbox{\strut \textbf{Plant Dynamics}} }}
         &{\vtop{\hbox{\strut \textbf{Discrete/Continuous} }}}
         &{\vtop{\hbox{\strut  \textbf{Workable Activation Functions}  }}} 
         &{\vtop{\hbox{\strut  \textbf{Core Techniques} }}}  
         &{\vtop{\hbox{\strut  \textbf{Model-Agnostic Solution} }}}  \\ \hline 
        \textbf{SMC}~\shortcite{sun2019formal} &  {Linear} & {Discrete} & {ReLU} &  {SMC encoding + solver}&{ \xmark} \\
        \hline
        \textbf{Verisig}~\shortcite{ivanov2019verisig} &  {Linear, Nonlinear } & {Discrete, Continuous } & { Sigmoid, Tanh   } &\makecell{ TM + Equivalent hybrid \\system transformation} &{ \xmark} \\
        \hline
        \textbf{ReachNN}~\shortcite{huang2019reachnn} &  {Linear, Nonlinear }& { Discrete, Continuous } & { ReLU, Sigmoid, Tanh} &  {TM+ Bernstein polynomial  }&{ \xmark} \\
        \hline
        \textbf{Sherlock}~\shortcite{dutta2019reachability}&  {Linear, Nonlinear}& {Discrete, Continuous} & {ReLU} &{TM + MILP} &{ \xmark} \\
        \hline
        \textbf{ReachNN*}~\shortcite{fan2020reachnn}&  {Linear, Nonlinear}& {Discrete, Continuous} & { ReLU, Sigmoid, Tanh } &  \makecell{TM + Bernstein polynomial\\ + parallel computation}&{ \xmark}\\
        \hline
        \textbf{Verisig 2.0}~\shortcite{ivanov2021verisig}& {Linear, Nonlinear} & {Discrete, Continuous} & { Tanh, Sigmod}  &\makecell{TM + Schrink wrapping\\+ Preconditioning}&{ \xmark}\\
        \hline
        \textbf{\makecell{DeepNNC\\(Our work)}}  & \makecell{Any Lipschitz\\ continuous systems} & {Continuous} & \makecell{ Any Lipschitz continuous activations\\e.g., ReLU, Sigmoid, Tanh, etc.}  &{ Lipschitz optimisation }  &{ \cmark} \\
     \bottomrule
  \end{tabular}}
\end{table*}

Recently, some pioneering research has emerged to verify NNCSs. They first estimate the output range of the controller and then feed the range to a hybrid system verification tool. Although both groups of tools can generate a tight estimate on their isolated models, direct combination results in a loose estimate of NNCS reachable sets after a few control steps, known as the {\it wrapping effect}~\cite{neumaier1993wrapping}. 

To resolve the wrapping effect, researchers develop various methods to analyse NNCSs as a whole.  Verisig~\cite{ivanov2019verisig} transforms the neural network into an equivalent hybrid system for verification, which can only work on Sigmoid activation function. NNV~\cite{tran2020nnv} uses a variety of set representations, e.g., polyhedra and zonotopes, for both the controller and the plant. Other researchers are inspired by the fact that the Taylor model approach can analyse hybrid systems with traditional controllers~\cite{chen2012taylor}. Thus, they applied the Taylor model for the verification of NNCSs. Representatives are ReachNN~\cite{huang2019reachnn}, ReachNN*~\cite{fan2020reachnn}, Versig 2.0~\cite{ivanov2021verisig} and Sherlock~\cite{dutta2019reachability}. 
However, verification approaches based on Taylor polynomials are limited to certain types of activation function. And their over-approximation errors still exist and accumulate over time.

To alleviate the above weakness, this paper proposes a novel verification method for NNCSs, called DeepNNC, taking advantage of the recent advance in Lipschitzian optimisation. As shown in Table~\ref{tab-comparison}, compared to the state-of-the-art NNCS verification, DeepNNC can solve linear and nonlinear NNCS with a wide range of activation functions. 
Essentially, we treat the reachability problem as a black-box optimisation problem and verify the NNCS as long as it is Lipschitz continuous. DeepNNC is the only model-agnostic approach, which means that access to the inner structure of NNCSs is not required. DeepNNC has the ability to work on {\em any Lipschitz continuous complex system} to achieve efficient and tight reachability analysis.
The contributions of this paper are summarised below:

\begin{itemize}
\item This paper proposes a novel framework for the verification of NNCS, called DeepNNC. It is the first model-agnostic work that can deal with a broad class of hybrid systems with linear or nonlinear plants and neural network controllers with various types of activation functions such as ReLU, Sigmoid, and Tanh activation.

\item We theoretically prove that Lipschitz continuity holds on closed-loop NNCSs. 
DeepNNC constructs the reachable set estimation as a series of independent global optimisation problems, which can significantly reduce the over-approximation error and the wrapping effect. We also provide a theoretical analysis on the soundness and completeness of DeepNNC.

\item Our intensive experiments show that DeepNNC outperforms state-of-the-art NNCS verification approaches on various benchmarks in terms of both accuracy and efficiency. On average, DeepNNC is {\bf 768 times faster} than ReachNN*~\cite{fan2020reachnn}, {\bf 37 times faster} than Sherlock~\cite{dutta2019reachability}, and {\bf 56 times faster} than Verisig 2.0~\cite{ivanov2021verisig} (see Table~\ref{time consumption}).
\end{itemize}

\section{Related Work}\label{sec:related work}

We review the related works from the following three aspects.

\textbf{SMC-based approaches} transform the problem into an SMC problem~\cite{sun2019formal}. First, it partitions the safe workspace, i.e., the safe set, into imaging-adapted sets. After partitioning, this method utilises an SMC encoding to specify all possible assignments of the activation functions for the given plant and NNC. However, it can only work on discrete-time linear plants with ReLU neural controller.

\textbf{SDP-based and LP-based approaches}
SDP-based approach \cite{hu2020reach} uses a semidefinite program (SDP) for reachability analysis of NNCS. It
abstracts the nonlinear components of the closed-loop system by quadratic constraints and computes the approximate reachable sets via SDP. It is limited to linear systems with NNC. LP-based approach\cite{everett2021reachability} provides a linear programming-based formulation of NNCSs. It demonstrates higher efficiency and scalability than the SDP methods. However, the estimation results are less tight than SDP-based methods.

\textbf{Representation-based approaches} use the representation sets to serve as the input and output domains for both the controller and the hybrid system. The representative method is NNV~\cite{tran2020nnv}, which has integrated various representation sets, such as polyhedron ~\cite{tran2019safety,tran2019parallelizable}, star sets~\cite{tran2019star} and zonotop~\cite{singh2018fast}. For a linear plant, NNV can provide exact reachable sets, while it can only achieve an over-approximated analysis for a nonlinear plant.

\textbf{Taylor model based approaches} approximate the reachable sets of NNCSs with the Taylor model (TM). Representative methods are Sherlock~\cite{dutta2019reachability} 
and Verisig 2.0~\cite{ivanov2021verisig}. The core concept of the Taylor model is to approximate a function with a polynomial and a worst-case error bound. TM approximation has shown impressive results in the analysis of the reachability of hybrid systems with conventional controllers~\cite{chen2012taylor,chen2013flow} and the corresponding tool Flow*\cite{chen2013flow} is widely applied.
Verisig~\cite{ivanov2019verisig} makes use of the tool Flow*\cite{chen2013flow} by transforming the NNC into a regular hybrid system without a neural network. It proves that the Sigmoid is the solution to a quadratic differential equation and applies Flow*~\cite{chen2013flow} on the reachability analysis of the equivalent system.
Instead of directly using the TM verification tool, ReachNN~\cite{huang2019reachnn}, ReachNN*~\cite{fan2020reachnn} and Verisig 2.0~\cite{ivanov2021verisig} maintain the structure of NNCS and approximate the input and output of NNC by a polynomial and an error band, with a structure similar to the TM model. Therefore, the reachable sets of controller and plant share the same format and can be transmitted and processed in the closed loop. 

Table~\ref{tab-comparison} presents a detailed comparison of DeepNNC with its related work. The core technique of our method is significantly different from existing solutions, enabling DeepNNC to work on any neural network controlled system as long as it is Lipschitz continuous.

\section{Problem Formulation}
\label{problem formulation}

This section first describes the NNCS model and then formulates the NNCS reachability problem. We deal with a closed-loop neural network controlled system as shown in Figure \ref{nncs}. It consists of two parts, a plant represented by a continuous system $\dot{x}=f(x,u)$ and a neural network controller $u=\sigma(y)$. The system works in a time-triggered manner with the control step size $\delta$.

\begin{figure}[ht]
\begin{center}
\centerline{\includegraphics[width=0.65\columnwidth]{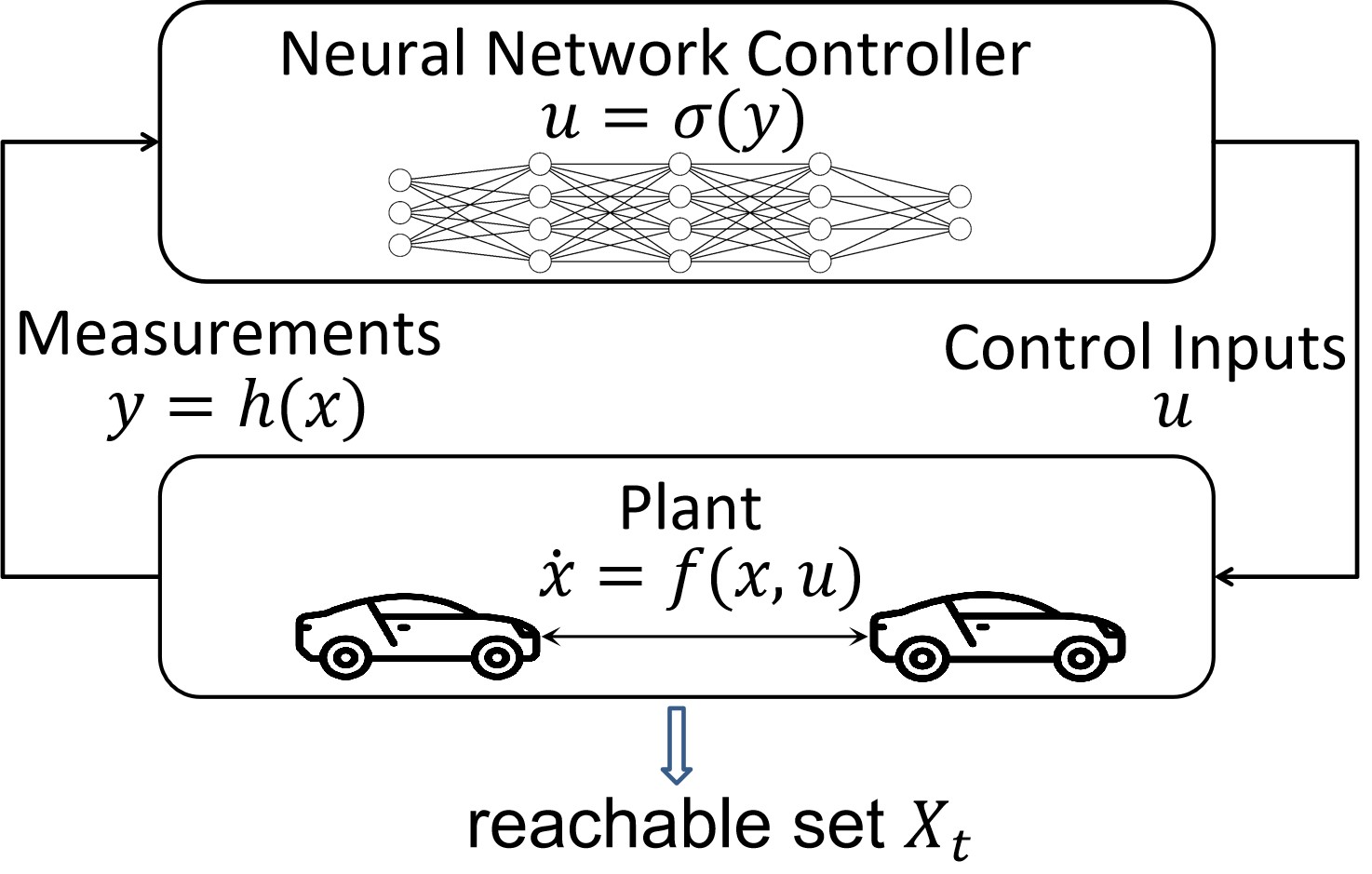}}
\caption{Architecture of a neural network controlled system. {\it Neural network controller $\delta$ takes measurements $y$ as input and initiates the control input $u$. The plant updates its states according to its dynamic $\dot{x}=f(x,u)$ and generates new measurements of the states to be fed to the controller.}}
\label{nncs}
\end{center}
\vspace{-8mm}
\end{figure}


\begin{definition}[Plant]\label{def-plant}
{\it
In this work, we specify the plant $P$ as a continuous system as:
\begin{equation}\small
    \dot{x}=f(x,u), \quad  \ x\in\mathbb{R}^n \ \text{and} \ u\in\mathbb{R}^m
\end{equation}
with state $x$ containing $n$ variables, 
and control input $u$ consisting of $m$ variables.
~For the time-triggered controller, the plant dynamic inside a control time step is assigned as
\begin{equation}\small
\label{dynamic interval}
       \dot{x}(t)=f(x(t),u(i\delta)) 
\end{equation}
with $  i=0,1,2,..., $ and $ t\in[i\delta, (i+1)\delta] $.
}
\end{definition}
In control systems, the plant is usually modelled in the form of an ordinary differential equation (ODE). To guarantee the existence of a unique solution in ODE, the function $f$ to model the plant is required to be Lipschitz continuous in $x$ and $u$ \cite{meiss2007differential}.



\begin{definition}[Neural Network Controller]\label{def-nnc}
{\it
The neural network controller is a $k$ layer feedforward neural network with one or multiple types of activation functions, such as Sigmoid, Tanh, and ReLU activations, which can be defined as:
\begin{equation}\small
    \sigma(y)= \sigma_{k}\circ\sigma_{k-1}\circ...\circ\sigma_{1}(y).
\end{equation} 
}
\end{definition}

As shown in Figure~\ref{nncs}, system transmits the measurements $y$ to the NNC, forming the control input through $u=\sigma(y)$. The measurement function $y=h(x)$ is usually Lipschitz continuous, and some systems use $y=x$ for simplicity \cite{ge2013stable}.

{\it 
\begin{definition}[Neural Network Controlled System]
\label{nncs def}
The neural network controlled system $\phi$ is a closed-loop system composed of a plant $P$ and a neural network controller $\sigma$. In a control time step, the closed-loop system is characterised as:
\begin{equation}\small
\begin{split}
    \phi(x(i\delta),\Delta t)  & \equiv\ x(i\delta+\Delta t) \\ =  x(i\delta) & +\int_{i\delta}^{i\delta+\Delta t} \ f(x(\tau), \sigma(h(x(i\delta)))) \ d \tau
    \end{split}
\end{equation}
with $i=0,1,2,3..., $ and $ \Delta t\in[0, \delta]$. For a given initial state $x_0\in\mathcal{X}_0$ at $t=0$ and a specified NNCS $\phi$, the state $x$ at time $t\ge 0$ is denoted as $x(t)\equiv\phi(x_0,t)$.
\end{definition}
}
{\it
\begin{definition}[Reachable Set of NNCS]
Given a range $\mathcal{X}_0$ for initial states and a neural network controlled system $\phi$, we define the region $S_t$ a reachable set at time $t\ge 0$ if all reachable states can be found in $S_t$.
\begin{equation}\small
{\forall}x_0\in \mathcal{X}_0, \ {\exists} \ \phi(x_0,t)\in S_t
\end{equation}
\end{definition}
}
It is possible that an over-approximation of reachable states exists in the reachable set. In other words, all reachable states of time $t$ stay in $S_t$, while not all states in $S_t$ are reachable. Furthermore, we can formulate the estimation of the reachable set of NNCS as an optimisation problem:

\begin{definition}[Reachability of NNCS]
Let the range $\mathcal{X}_0$ be initial states $x_0$ and $\phi$ be a NNCS. The reachability of NNCS at a predefined time point $t$ is defined as the reachable set $S_t(\phi,X_0,\epsilon) = [l,u]$ of NNCS $\phi$ under an error tolerance $\epsilon \geq 0$ such that 
\begin{equation}\label{eqn:pr-1}
\begin{split}
   \inf_{x_0 \in X_0} \phi(x_0,t) - \epsilon \leq l \leq  \inf_{x_0 \in X_0} \phi(x_0,t)  + \epsilon 
 \\ \sup_{x_0 \in X_0} \phi(x_0,t) - \epsilon  \leq  u \leq \sup_{x_0 \in X_0} \phi(x_0,t) + \epsilon 
 \end{split}
\end{equation}
%
\end{definition}
As discussed in multiple works~\cite{ruan2018reachability,katz2017reluplex}, the reachability problem in neural networks is extremely challenging, it is an NP-complete problem.

\section{Reachability Analysis via Lipschitzian Optimisation}
\label{lcnncs}

This section presents a novel solution taking advantage of recent advances in Lipschitzian optimisation~\cite{gergel2016adaptive,huang2022bridging}. But before employing Lipschitzian optimisation, we need to theoretically prove the Lipschitz continuity of the NNCS.

\subsection{Lipschitz Continuity of NNCSs}

\begin{theorem}[Lipschitz Continuity of NNCS]\label{lip_continue_nncs}
Given an initial state $X_0$ and a neural network controlled system $\phi(x,t)$ with controller $\sigma(x)$ and plant $f(x,u)$,  if $f$ and $\sigma$ are Lipschitz continuous, then the system $\phi(x,t)$ is Lipschitz continuous in initial states $\mathcal{X}_0$. A real constant $K\geq0$ exists for any time point $t>0$ and for all $x_0, y_0 \in \mathcal{X}_0$:
\begin{equation}\small
\label{lip_continue_nccs}
  \left|\phi(x_0,t)-\phi(y_0,t)\right| \le K \left|x_0-y_0 \right|
\end{equation}
The smallest $K$ is the best Lipschitz constant for the system, denoted $K_{best}$.
\end{theorem}
\begin{proof}\let\qed\relax(sketch)
The essential idea is to open the control loop and further demonstrate that the system $\phi(x_0,t)$ is Lipschitz continuous on $x_0$. We duplicate the NNC and plant blocks to build an equivalent open-loop control system. 
Based on Definition \ref{nncs def}, we have
\begin{equation}\small
\label{diff}
\begin{split}
 & \left |\phi(x_0,t)- \phi(y_0,t) \right |  \le   \left |x_0-y_0\right | \\ + & \int_{t_0}^{t}  \left| \ f(x(\tau), u_x(\tau)) \ -  \ f(y(\tau), u_y(\tau)) \right| d \tau  
\end{split}
\end{equation}
We add the term $f(x(\tau), u_y(\tau)) \ - \ f(x(\tau), u_y(\tau))$ to the right side and assume that the Lipschitz constant of $f$ on $x$ and $u$ are $L_u$ and $L_x$ and the Lipschitz constant of NN as $L_n$, we can transform Equation~(\ref{diff}) to the following form:
\begin{equation}\small
\begin{split}
    &  \left\|\phi(x_0,t) - \phi(y_0,t) \right\| \le \\ 
   & ( L_u L_n (t-t_0) +1 ) \left\|x_0-y_0\right\| 
   +   L_x \int_{t_0}^{t} \left\| \  x(\tau) - y(\tau) \  \right\| \  d \tau.
\end{split}
\end{equation}
Based on Grönwall's theory of integral inequality \cite{gronwall}, we further have the following equation:
\begin{equation}\small
\label{1_step}
\begin{split}
& \left\|\phi(x_0,t)- \phi(y_0,t) \right\|  \le  ( L_u L_n (t-t_0) +1 )  \left\|x_0-y_0\right\| e^{L_x(t-t_0)}. 
\end{split}
\end{equation}
Hence, we have demonstrated the Lipschitz continuity of the closed-loop system on its initial states with respect to the time interval $[0,\delta]$. We repeat this process through all the control steps and prove the Lipschitz continuity of NNCSs on initial states. See {\bf Appendix-A}\footnote{All appendixes of this paper can be found at \url{https://github.com/TrustAI/DeepNNC/blob/main/appendix.pdf}} for a detailed proof.
~~~\qedsymbol{}
\end{proof}

Based on Theorem~\ref{lip_continue_nncs}, we can easily have the following lemma about the local Lipschitz continuity of the NNCS.
\begin{lemma}[Local Lipschitz Continuity of NNCS]
\label{local_lip}
If NNCS $\phi(x,t)$ is Lipschitz continuous throughout $X_0$, then it is also locally Lipschitz continuous in its sub-intervals. There exists a real constant $k_i\geq0$ for any time point $t>0$, and $\forall x_0, y_0 \in[a_i,a_{i+1}]\subset X_0$, we have
$  \left\|\phi(x_0,t)-\phi(y_0,t)\right\| \le k_i \left\|x_0-y_0 \right\|  $.
\end{lemma}
Based on Lemma~\ref{local_lip}, we develop a fine-grained strategy to estimate local Lipschitz constants, leading to faster convergence in Lipschitz optimisation.

\subsection{Lipschitzian Optimisation}

In the reachability analysis, the primary aim is to provide the lower bound $R$ of the minimal value $\phi(x),x\in X_0$. We achieve tighter bounds by reasonably partitioning the initial input region $X_0$ and evaluating the points at the edges of the partition. The pseudocode of the algorithm can be found in {\bf Appendix-B}. Here, we explain the partition strategy and the construction of $R$ in the $k$-th iteration.

\begin{figure}[t]
\begin{center}
\centerline{\includegraphics[width=0.9\columnwidth]{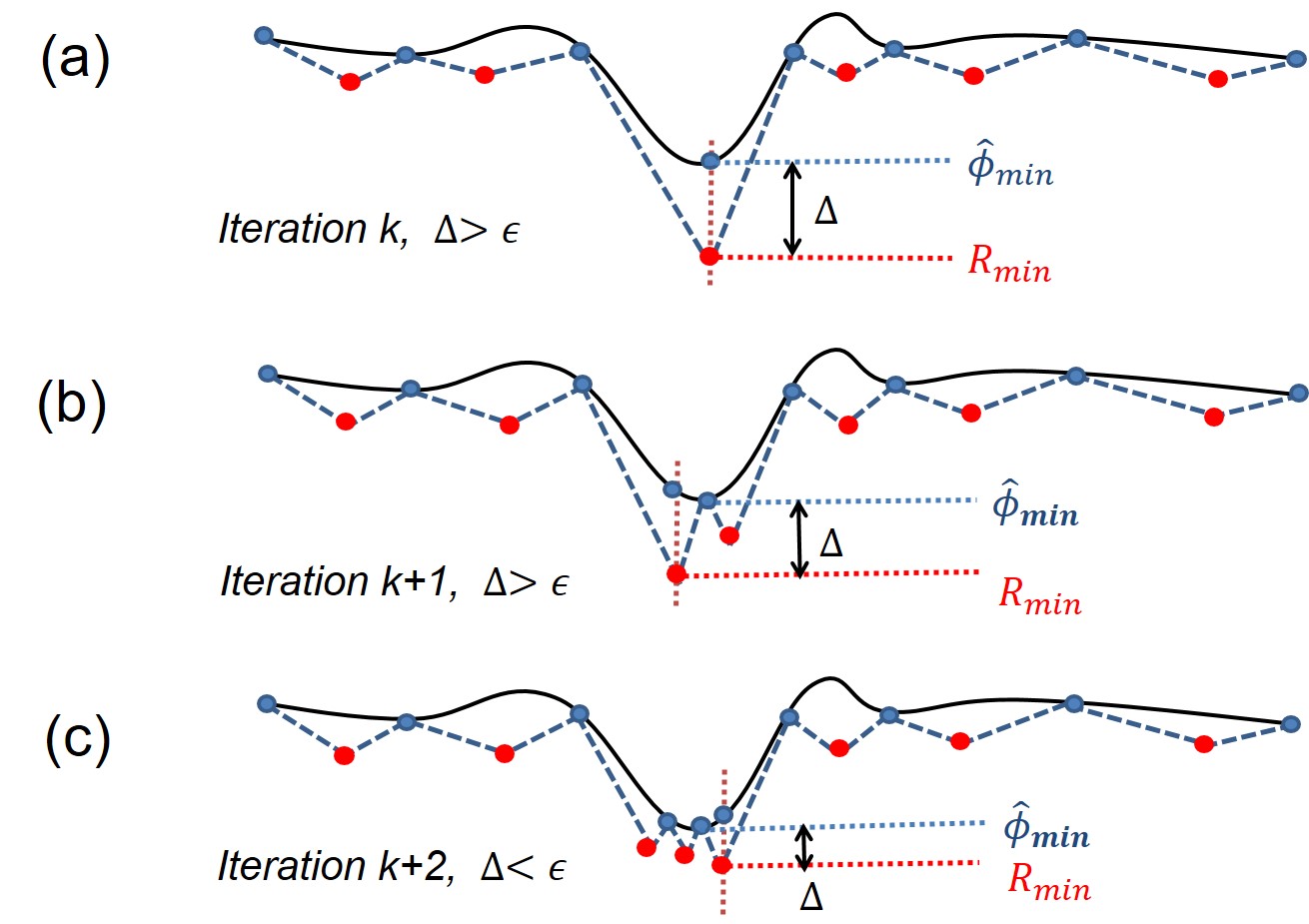}}
\caption{Demonstration of the optimisation process: {\it $|\hat{\phi}_{min}^k-R_{min}^k|$ decreases with iteration $k$, optimisation converges when $|\hat{\phi}_{min}^k-R_{min}^k|<\epsilon$}}
\label{lip_opt}
\end{center}
\vspace{-8mm}
\end{figure}

\begin{itemize}[leftmargin=*]
\item[i)] We sort the input sub intervals $D_0,D_1,D_2...D_n$, with  $D_i=[a_i, a_{i+1}]$. In the first iteration, we only have one sub interval with $D_0=X_0=[a_0,a_1]$. 
\item[ii)] 
We evaluate the system output $\phi_i=\phi_t(a_i)$ at the edge points $a_0, ...a_n$, and define $R_i$ as the character value of the interval $D_i$ with the corresponding input $x^{R}_i$:
\begin{equation}\small
     R_{i}=\frac{\phi_{i}+\phi_{i+1}}{2}+l_i\frac{a_{i}-a_{i+1}}{2}
\end{equation}
\begin{equation}\small
    x^R_{i}=\frac{\phi_{i}-\phi_{i+1}}{2l_i}+\frac{a_{i+1}+a_{i}}{2}.
\end{equation}
Both $R_i$ and $x^{R}_i$ are determined by the information of the edge point 
and local Lipschitz constant $l_i$.
\item[iii)] We search for the interval with a minimum character value and minimal evaluated edge points:
\begin{equation}\small
\begin{split}
      &   R_{min}^k=\min\{R_0,...,R_n\}, \ j=\arg \min \{R_0,...,R_n\} \\
    & D_j=[a_j,a_{j+1}] , \ \hat{\phi}_{min}^k=\min\{\phi_0,\phi_1,...\phi_n\}.
     \end{split}
\end{equation}
If $\hat{\phi}_{min}^k-R_{min}^k \le \epsilon$, the optimisation terminates, and we return $R_j$ for the lower bound. Otherwise, we add a new edge point $a_{new}= x^R_{j}$ to divide $D_j$.
\end{itemize}

As illustrated in Figure \ref{lip_opt}, with blue points as edges and red points indicating the characteristic value $R$. Both $R_{min}^k$ and $\hat{\phi}_{min}^k$ are approaching the real minimum value $\phi_{min}$. The optimisation ends when $\hat{\phi}_{min}^k-R_{min}^k \le \epsilon$, which implies $ R_{min}^k\le \phi_{min} \le R_{min}^k+\epsilon$

To extend the above strategy into a multi-dimensional case, we use the nested optimisation scheme to solve the problem in a recursive way.
\begin{equation}\small
\min\limits_{x \in [p_i,q_i]^n}~~\phi(x) =
\min\limits_{x_1\in [p_1,q_1]}... \min\limits_{x_n\in [p_n,q_n]} \phi(x_1,...,x_n)
\end{equation}
We define the $d$-th level optimisation sub-problem,
\begin{equation}\small
w_d(x_1,...,x_d) = \min_{x_{d+1}\in [p_{d+1},q_{d+1}]} w_{d+1}(x_1,..., x_{d+1}) 
\end{equation}
		and for $d=n$, $w_n(x_1,...,x_n) = \phi(x_1,x_2,...,x_n).$
Thus, we have $\min_{{x} \in [p_i,q_i]^n}~~\phi({x}) = \min_{x_1\in [p_1,q_1]} \phi_1(x_1)$ which is actually a one-dimensional optimisation problem.


\subsection{Convergence Analysis}

We discuss the convergence analysis in two circumstances, i.e., one-dimensional problem and multidimensional problem. In the one-dimensional problem,
we divide the sub-interval $D_j=[a_j,a_{j+1}]$ into $D_{new}^1=[a_j,x_j^R]$ and $D_{new}^2=[x_j^R, a_{j+1}]$ in the $k$-th iteration. The new character value
 $    R_{new}^1-R_j=l_j\frac{a_{j+1}-x_{j}^R}{2}-\frac{\phi_{j+1}-\phi(x_j^R)}{2}>0$ ,
  $   R_{new}^2-R_j=l_j\frac{x_{j}^R-a_{j}}{2}-\frac{\phi(x_j^R)-\phi_{j}}{2}>0$.
Since $R_{min}^k=\min\{R_0,...R_n\}\setminus\{R_j\}\}\cup\{R_{new}^1, R_{new}^2\}$, we confirm that $R_{min}^k$ increases strictly monotonically and is bounded. 

In the multidimensional problem, the problem is transformed to a one-dimensional problem with nested optimisation.  We introduce Theorem~\ref{inductive} for the inductive step.

\begin{theorem}\label{inductive}
In optimisation, if $\forall {x}\in \mathbb{R}^d$, $\lim_{k\to \infty}R_{min}^k = \inf_{{x}\in [a,b]^d} \phi({x})$ and $\lim_{i\to \infty}(\hat{\phi}_{min}^k-R_{min}^k) = 0$ are satisfied, then $ \forall x\in \mathbb{R}^{d+1}$, $\lim_{k\to \infty}R_{min}^k = \inf_{{x}\in [a,b]^{d+1}} \phi({x})$ and $\lim_{k\to \infty}(\hat{\phi}_{min}^k-R_{min}^k) = 0$ hold.
\end{theorem}
\begin{proof}\let\qed\relax(sketch)
By the nested optimisation scheme, we have
\begin{equation}\small
    	\min_{\mathbf{x} \in [a_i,b_i]^{d+1}}\phi(\mathbf{x}) =\min_{x \in [a,b]} W(x); ~~~
	W(x) = \min_{\mathbf{y} \in [a_i,b_i]^d} \phi(x,\mathbf{y})
\end{equation}
	Since $\min_{\mathbf{y} \in [a_i,b_i]^d} \phi(x,\mathbf{y})$ is bounded by an interval error $\epsilon_{\mathbf{y}}$, then we have
	 $|W(x) - W^*(x)|\leq \epsilon_{\mathbf{y}}, \forall x \in [a,b]$, where $W^*(x)$ is the accurate evaluation of the function.
  
For the inaccurate evaluation case, we have $W_{min} = \min_{x\in [a,b]} W(x)$,  $R_{min}^{k}$ and $\widehat{W}_{min}^k$. The termination criteria for both cases are $|\widehat{W}^k_{min}-R_{min}^{k}|\leq \epsilon_x$ and $|\widehat{W}^{k*}_{min}-R_{min}^{k*}|\leq \epsilon_x$, and $w^*$ represents the ideal global minimum. 
	Based on the definition of $R_{min}^k$,  we have $w^* -R_{min}^k \leq \epsilon_{\mathbf{y}} +\epsilon_x$. By analogy, we get $\widehat{W}^k_{min}-w^* \leq \epsilon_{\mathbf{y}} +\epsilon_x$. Thus, the accurate global minimum is bounded. Theorem~\ref{inductive} is proved. See {\bf Appendix C} for a detailed proof. ~~~\qedsymbol{}
\end{proof}



\subsection{Estimation of Lipschitz Constant}\label{lip_estimation}

To enable a fast convergence of the optimisation, we need to provide a Lipschitz constant for NNCS. As we aim for a model-agnostic reachability analysis for NNCSs, we further propose two practical variants to estimate the Lipschitz constant for a black-box NNCS.


\subsubsection{Dynamic Estimation of Global Lipschitz Constant}

In the optimisation process, the entire input range $X$ of the function $\phi_t(x)$ is divided into limited sub-intervals $D_0,D_1,D_2...D_n$, where $D_i=[a_i, a_{i+1}]$. We update the global Lipschitz constant $L$ according to the interval partitions:
\begin{equation}\small
  L= r\cdot max \bigg\lvert \frac{\phi_t(a_{i+1})-\phi_t(a_i)}{a_{i+1}-a_i} \bigg\rvert
\end{equation}
To avoid an underestimate of $L$, we choose $r>1$ and have
\begin{equation}\small
    \lim_{j \to \infty} r\cdot \max_{i = 1,...,j-1} \bigg\lvert \frac{\phi_t(a_{i+1})-\phi_t(a_i)}{a_{i+1}-a_i} \bigg\rvert = r \cdot\sup_{a\in X} {\dfrac{d\phi_t}{da}} > K_{best}.
\end{equation}
The dynamic update of $L$ will approximate the best Lipschitz constant of the NNCS.




\subsubsection{Local Lipschitz Constant Estimation} \label{local_lip_estimation}
An alternative to improve efficiency is the adoption of the local Lipschitz constant $l_i$.
~We adopt the local adjustment \cite{sergeyev1995information} for local Lipschitz estimation, which considers not only global information, but also neighbourhood intervals. For each sub-interval, we introduce $m_i$
\begin{equation}\small
    m_i= | \phi_t(a_{i+1})-\phi_t(a_i)|/|a_{i+1}-a_i| , M=\max \ m_i
\end{equation}
We calculate sub-interval sizes $d_i$ and select the largest $D$.
\begin{equation}\small
d_i=\left|a_{i+1}-a_i\right|, 
D= \max \ d_i
\end{equation}

Equation (\ref{local_lip_cons}) estimates the local Lipschitz constant.
\begin{equation}\small
\label{local_lip_cons}
    l_i=r \cdot \max \left\{ m_{i-1}, m_i, m_{i+1}, M \cdot {d_i} /{D} \right\}
\end{equation}
It balances the local and global information. When the sub-interval is small, that is, $ {M\cdot d_i/D} \to 0$, $l_i$ is decided by local information. \begin{small}$$\lim_{d_i \to 0} l_i=r \cdot max \left\{ m_{i-1}, m_i, m_{i+1}\right\}=r \cdot\sup_{a\in [a_i,a_{i+1}]} {\dfrac{d\phi_t}{da}}>k_i^{best}$$\end{small}
When the subinterval is large, that is, $d_i \to D$, local information is not reliable. $l_i$ is determined by global information.
\begin{equation}\small
    \lim_{d_i \to D} l_i=r \cdot M=r \cdot\sup_{a\in X} {\dfrac{d\phi_t}{da}}>K_{best}>k_i^{best}
\end{equation}
Local Lipschitz constants provide a {\it more accurate} polyline approximation to the original function, leading to faster convergence.

\begin{figure}[t]
\begin{center}
\centerline{\includegraphics[width=0.5\columnwidth]{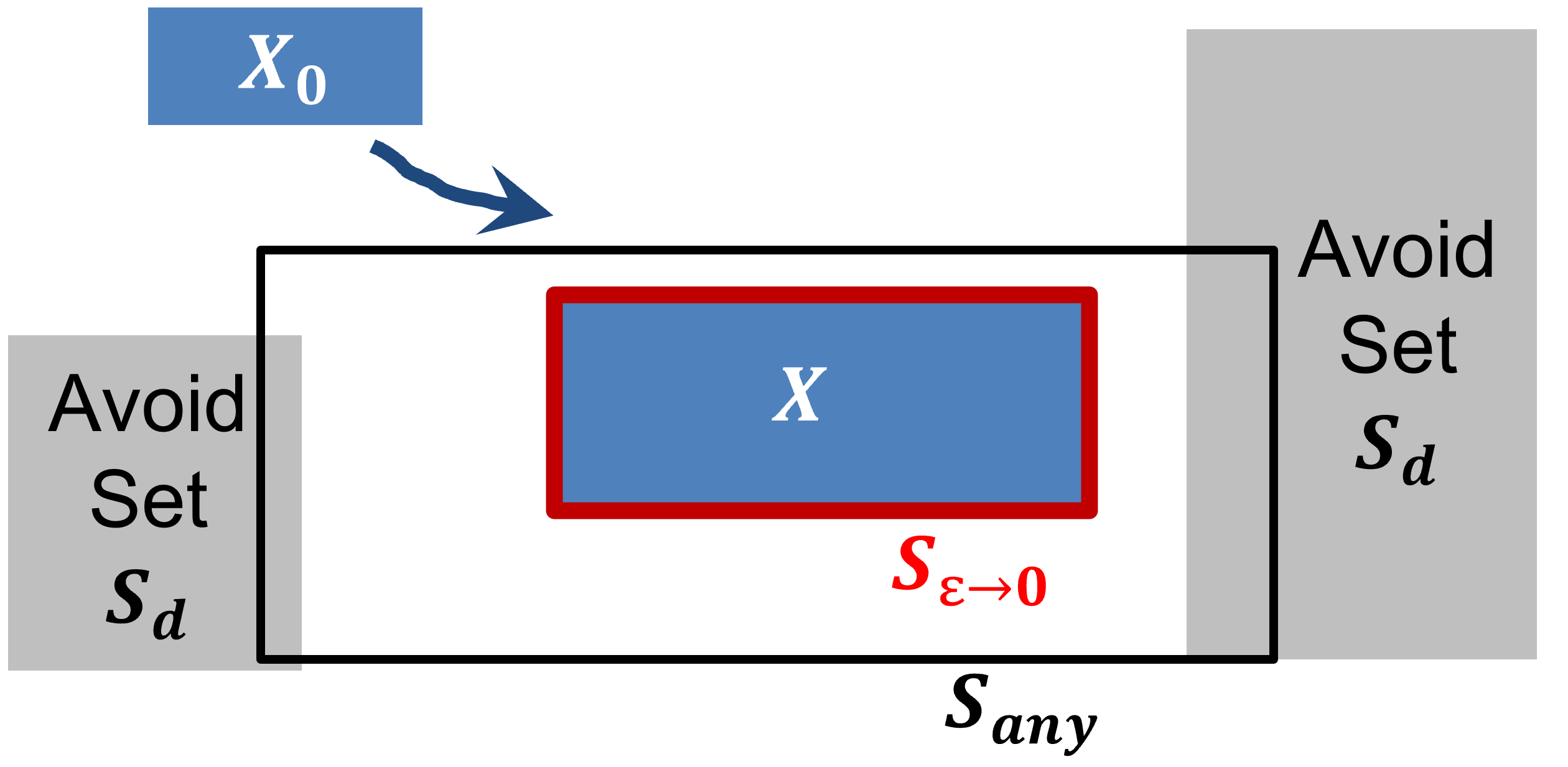}}
\caption{$X_0$ is the initial state set and $S_d$ is the avoid set. DeepNNC  can return an overestimation $S_{any}$ in any iteration. Verification is sound but not complete by $S_{any}$, and Verification is both sound and complete by $S_{\epsilon \to 0}$.}
\label{overestimation}
\end{center}
\vspace{-4mm}
\end{figure}

\subsection{Soundness and Completeness}
\begin{theorem}[Soundness]
Given a box-constrained initial input range $X_0$ and a dangerous set $S_a$, the verification of an NNCS via DeepNNC  is sound anytime. DeepNNC can return an overestimation $S$ of the real reachable set $X$ at any iteration, even it has not reached the convergence.
\end{theorem}

\begin{proof}\let\qed\relax
When we interrupt the optimisation at any iteration $k$, lower bound $R_{min}^k$ is returned for verification. We assume that the real minimal value $\phi_{min}$ is located in the sub-interval $[a_i,a_{i+1}]$, i.e. $\phi_{min}=\phi(x_{min})$ with $x_{min} \in[a_i,a_{i+1}]$.
\begin{small}
\begin{equation}
\begin{split}
 & R_i-\phi_{min}=\frac{\phi_{i}+\phi_{i+1}}{2}-l_i\frac{a_{i}-a_{i+1}}{2}-\phi_{min}
 \\=&\frac{\phi_{i}-\phi_{min}}{2} +l_i\frac{a_{i}-x_{min}}{2} +\frac{\phi_{min}-\phi_{i+1}}{2}+l_i\frac{x_{min}-a_{i+1}}{2} \le 0
\end{split}
\end{equation}
\end{small}
Thus, $R_{min}^k-\phi_{min}\le R_i - \phi_{min}\le 0$ holds.
~When using $R_{min}^k$ to estimate the output range $S_{any}$, $R_{min}^k\le  \phi_{min}$ ensures $X \subset S_{any}$, where $X$ is the real reachable set.  Given an avoid set $S_d$, if $S_{any}\cap S_d = \emptyset$, then $X\cap S_d = \emptyset$, the NNCS is safe. Thus, the approach DeepNNC is sound.~\qedsymbol{}
\end{proof}

\begin{theorem}[Completeness]
Given a box-constrained initial input range $X_0$ and a dangerous set $S_d$, verification via DeepNNC is complete only when optimisation reaches convergence with the overestimation error $\epsilon \to 0$.
\end{theorem}
\begin{proof}\let\qed\relax
As demonstrated in the convergence analysis, with iteration number $k\to \infty$ the optimisation can achieve convergence with error $\epsilon \to 0$. In such circumstances, the lower bound $R_{min}^k \to \phi_{min}$ results in an estimation $S_{\epsilon \to 0} \to X_0$ with ignorable estimation error. If $S_{\epsilon \to 0}\cap S_d \neq \emptyset$, then $X\cap S_d \neq \emptyset$, the NNCS is not safe. ~~~\qedsymbol{}
\end{proof}


\section{Experiments}
\label{experiments}
We first compare DeepNNC with some state-of-the-art baselines. 
We then analyse the influences of parameter $k_{max}$ and $\epsilon$. Finally, we provide a case study on an airplane model. An additional case study of an adaptive cruise control system is presented in {\bf Appendix D}. 

\subsection{Comparison with State-of-the-art Methods}

We test DeenNNC and baseline methods on six benchmarks. See {\bf Appendix-E} for the details of the benchmarks. We compare our method with ReachNN*~\cite{fan2020reachnn}, Sherlock~\cite{sherlock} and Versig 2.0~\cite{ivanov2021verisig} in terms of efficiency. Regarding the accuracy comparison, we compare our method with Verisig~\cite{ivanov2019verisig} instead of Sherlock~\cite{sherlock}, since the results of Sherlock loose fast with the increase of control steps.

\begin{table}[t]
\centering
\caption{Computation time to estimate the reachable set {\it Tools are tested on six benchmarks and four activation functions. Results are not given if a method is not applicable or reaches the time limitation.}}
\label{time consumption}
\footnotesize
\setlength{\tabcolsep}{0.5mm}
\scalebox{0.8}{
\begin{tabular}{|c|c|c|c|c|c|}
\hline
 & Controller & \textbf{DeepNNC(s)} & ReachNN*(s) & Sherlock(s) & Verisig2.0(s) \\ \hline
\multirow{4}{*}{1} & ReLU       & \textbf{0.62}          & 26       & 42       & -          \\ \cline{2-6} 
                   & Sigmoid    & \textbf{0.45}          & 75       & -        & 47         \\ \cline{2-6} 
                   & Tanh       & \textbf{0.62}        & 76  & -        & 46         \\ \cline{2-6} 
                   & ReLU+Tanh  & \textbf{0.60}         & 71       & -        & -          \\ \hline
\multirow{4}{*}{2} & ReLU       & \textbf{0.42}        & 5        & 3        & -          \\ \cline{2-6} 
                   & Sigmoid    & \textbf{0.50}        & 13       & -        & 7          \\ \cline{2-6} 
                   & Tanh       & \textbf{0.52}        & 73       & -        & unknown    \\ \cline{2-6} 
                   & ReLU+Tanh  & \textbf{0.54}        & 8        & -        & -          \\ \hline
\multirow{4}{*}{3} & ReLU       & \textbf{0.50}        & 94       & 143      & -          \\ \cline{2-6} 
                   & Sigmoid    & \textbf{0.55}        & 146      & -        & 44         \\ \cline{2-6} 
                   & Tanh       & \textbf{0.55}        & 137      & -        & 38         \\ \cline{2-6} 
                   & ReLU+Tanh  & \textbf{0.49}        & unknown  & -        & -          \\ \hline
\multirow{4}{*}{4} & ReLU       & \textbf{0.86}        & 8        & 21       & -          \\ \cline{2-6} 
                   & Sigmoid    & \textbf{1.02}        & 22       & -        & 11         \\ \cline{2-6} 
                   & Tanh       & \textbf{1.08}        & 21       & --       & 10         \\ \cline{2-6} 
                   & ReLU+Tanh  & \textbf{0.98}        & 12       & -        & -          \\ \hline
\multirow{4}{*}{5} & ReLU       & \textbf{0.70}        & 103      & 15       & -          \\ \cline{2-6} 
                   & Sigmoid    & \textbf{0.70}        & 27       & -        & 190        \\ \cline{2-6} 
                   & Tanh       & \textbf{0.71}        & unknown  & -        & 179        \\ \cline{2-6} 
                   & ReLU+Tanh  & \textbf{0.69}        & unknown  & -        & -          \\ \hline
\multirow{4}{*}{6} & ReLU       & \textbf{3.69}        & 1130     & 35       & -          \\ \cline{2-6} 
                   & Sigmoid    & \textbf{3.62}        & 13350    & -        & 83         \\ \cline{2-6} 
                   & Tanh       & \textbf{3.65}        & 2416     & -        & 70         \\ \cline{2-6} 
                   & ReLU+Tanh  & \textbf{3.64}        & 1413     & -        & -          \\ \hline
\end{tabular}}
\end{table}
The time consumption of different approaches for estimating the reachable set at a predefined time is demonstrated in Table \ref{time consumption} and Figure \ref{average_time}.
Our method can be applied to all benchmarks and has relatively better efficiency, especially for low-dimensional problems.

\begin{table}[!h]
\centering
\caption{Area size of reachable sets. We compare the area size of the polygon reachable set at a predefined point. A smaller area size indicates a tighter estimate. The percentage is calculated by reachable set area size of (baseline - ours) / baseline, i.e., improvement of tightness of estimation.}
\label{reachable set}
\begin{footnotesize}
\setlength{\tabcolsep}{2.2mm}
\scalebox{0.8}{
\begin{tabular}{|c|c|c|c|c|c|}
\hline
                   & \textbf{DeepNNC}   & Search     & Versig2.0   & Verisig    & ReachNN* \\ \hline
\multirow{2}{*}{\begin{tabular}[c]{@{}c@{}}B1\\Sigmoid\end{tabular}} & \textbf{7.7685}  &0.303     & 21      & NA         & 145   \\ \cline{2-6} 
                             & \textbackslash{} & \textbackslash{} & 63.01\%     & NA         & 94.64\%  \\ \hline
\multirow{2}{*}{\begin{tabular}[c]{@{}c@{}}B2\\Sigmoid\end{tabular}} & \textbf{12}     &5.7261      & 52      & 125     & 114   \\ \cline{2-6} 
                             & \textbackslash{} & \textbackslash{} & 76.92\%     & 90.4\%     & 89.47\%  \\ \hline
\multirow{2}{*}{\begin{tabular}[c]{@{}c@{}}B5\\Sigmoid\end{tabular}} & \textbf{3.4762}  &1.7139    & 3.6531  & 8.4109 & 98   \\ \cline{2-6} 
                             & \textbackslash{} &\textbackslash{} & 4.8\%       & 58.67\%    & 96.45\%  \\ \hline
\multirow{2}{*}{\begin{tabular}[c]{@{}c@{}}B5\\Tanh\end{tabular}}    & \textbf{3.4254}  &1.2747    & 3.8021  & 9.1202  & 78   \\ \cline{2-6} 
                             & \textbackslash{} & \textbackslash{} &7.28\%      & 62.44\%    & 95.6\%   \\ \hline
\multirow{2}{*}{\begin{tabular}[c]{@{}c@{}}B6\\Sigmoid\end{tabular}} & \textbf{6.2444}  &3.5401     & 6.5366 & 18.324  & 142   \\ \cline{2-6} 
                             & \textbackslash{} & \textbackslash{} & 4.47\%      & 65.92\%    & 95.6\%   \\ \hline
\multirow{2}{*}{{\begin{tabular}[c]{@{}c@{}}B6\\Tanh\end{tabular}}}    & \textbf{6.6674}  &3.636      & 7.2116   & 20.343  & 185   \\ \cline{2-6} 
                             & \textbackslash{} & \textbackslash{}  & 7.55\%      & 67.23\%    & 96.396\% \\ \hline
\end{tabular}}
\end{footnotesize}
\end{table}

\begin{figure}[!h]
\centering
\includegraphics[width=0.35\textwidth]{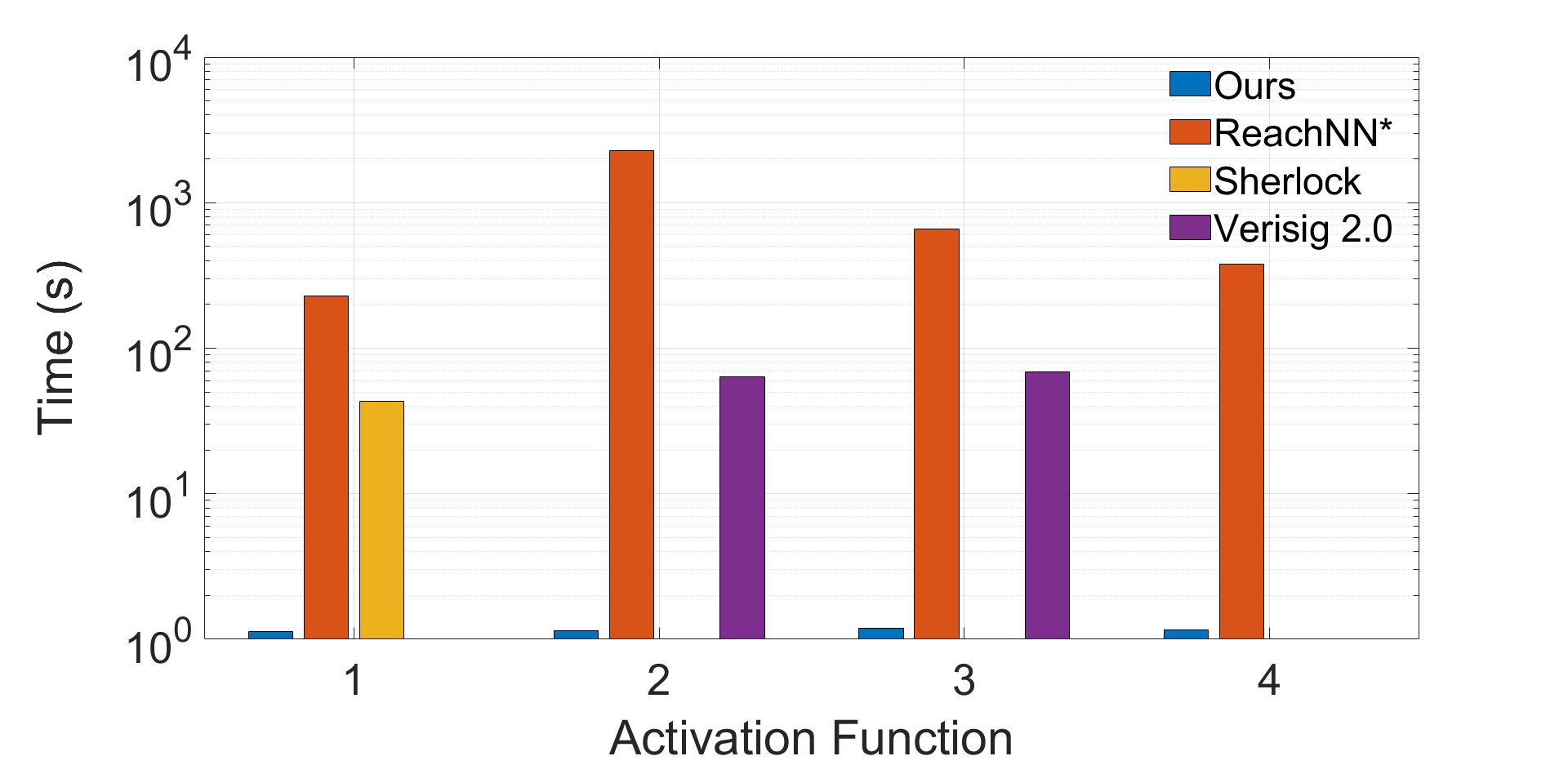}
\caption{Average time to verify different controllers. 1. ReLU 2. Sigmoid 3. Tanh 4. ReLU+Tanh. DeepNNC is {\bf 768 times} faster than ReachNN*,  {\bf 37 times} faster than Sherlock, {\bf 56 times} faster than Verisig 2.0.}
\label{average_time}
\vspace{-4mm}
\end{figure}
\begin{figure}[!h]
\vspace{-2mm}
\centering
\includegraphics[width=0.23\textwidth]{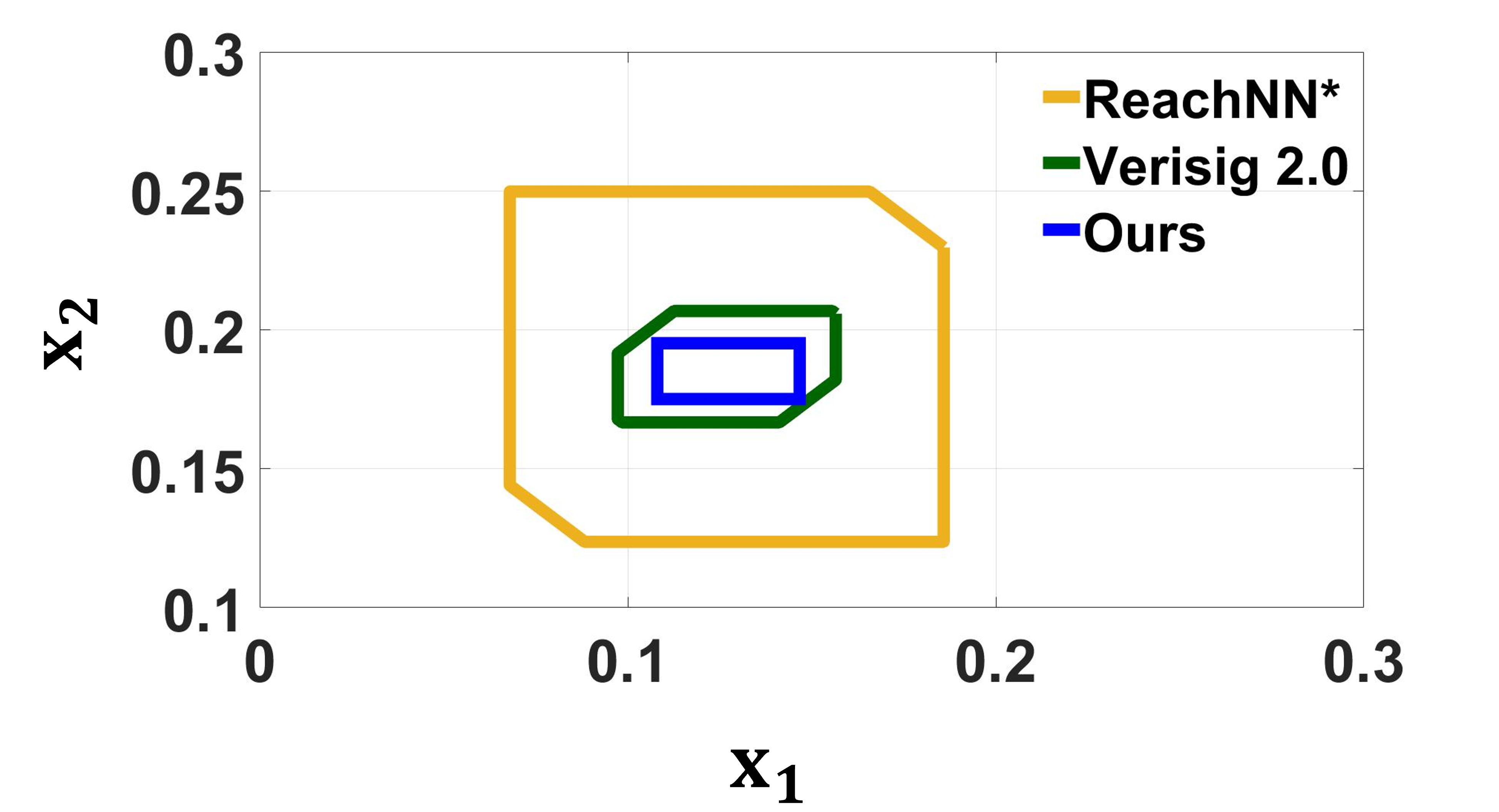}
\includegraphics[width=0.23\textwidth]{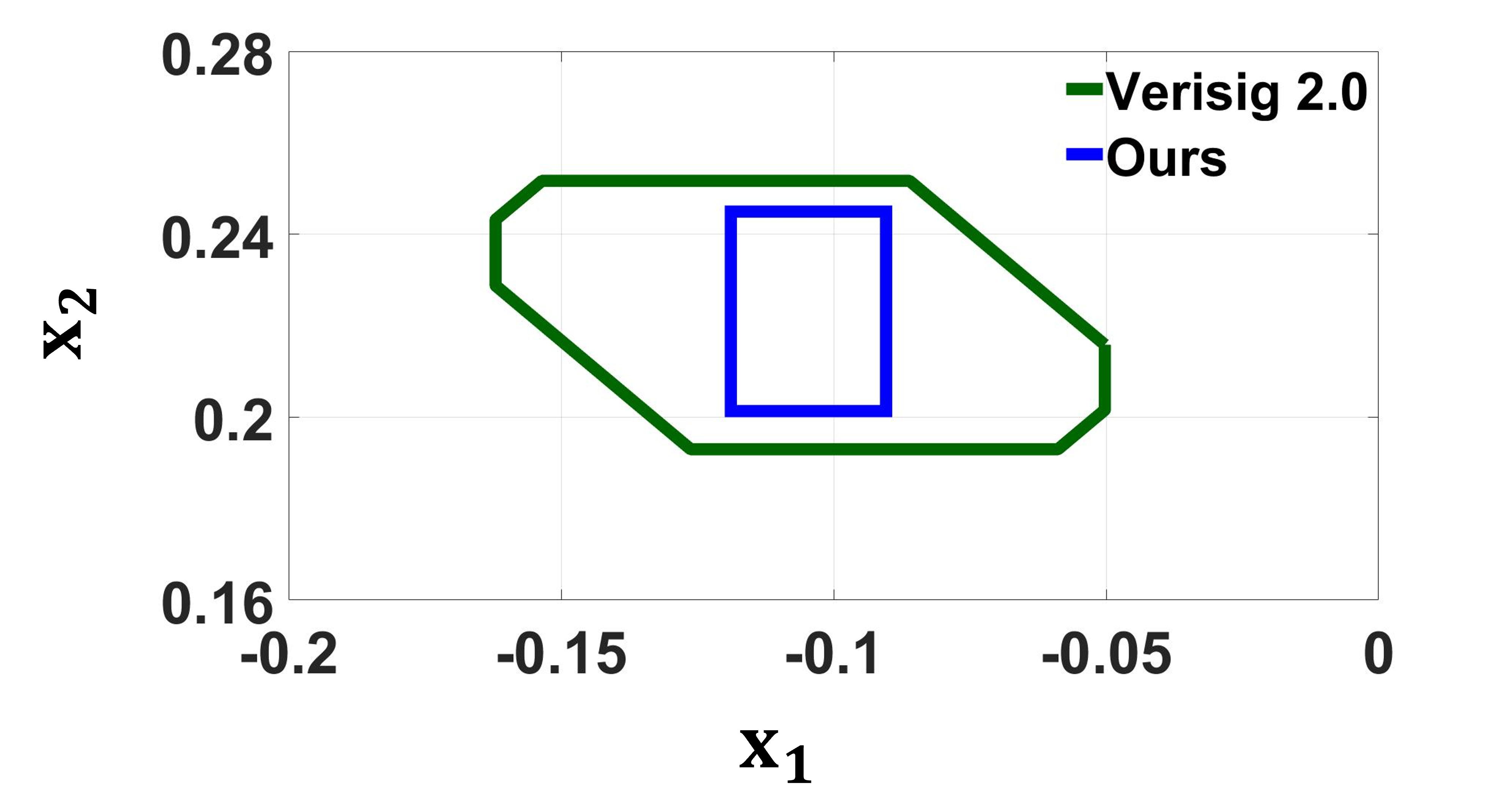}
\vspace{-3mm}
\caption{Reachable sets at a predefined time point estimated by different approaches. The estimated reachable set at a time point is a polygon. The size of the area of the reachable set indicates the precision of the estimation results.}
\label{final_reachable_set}
\vspace{-3mm}
\end{figure}

\begin{figure}[!ht]
\centering
\includegraphics[width=0.45\textwidth]{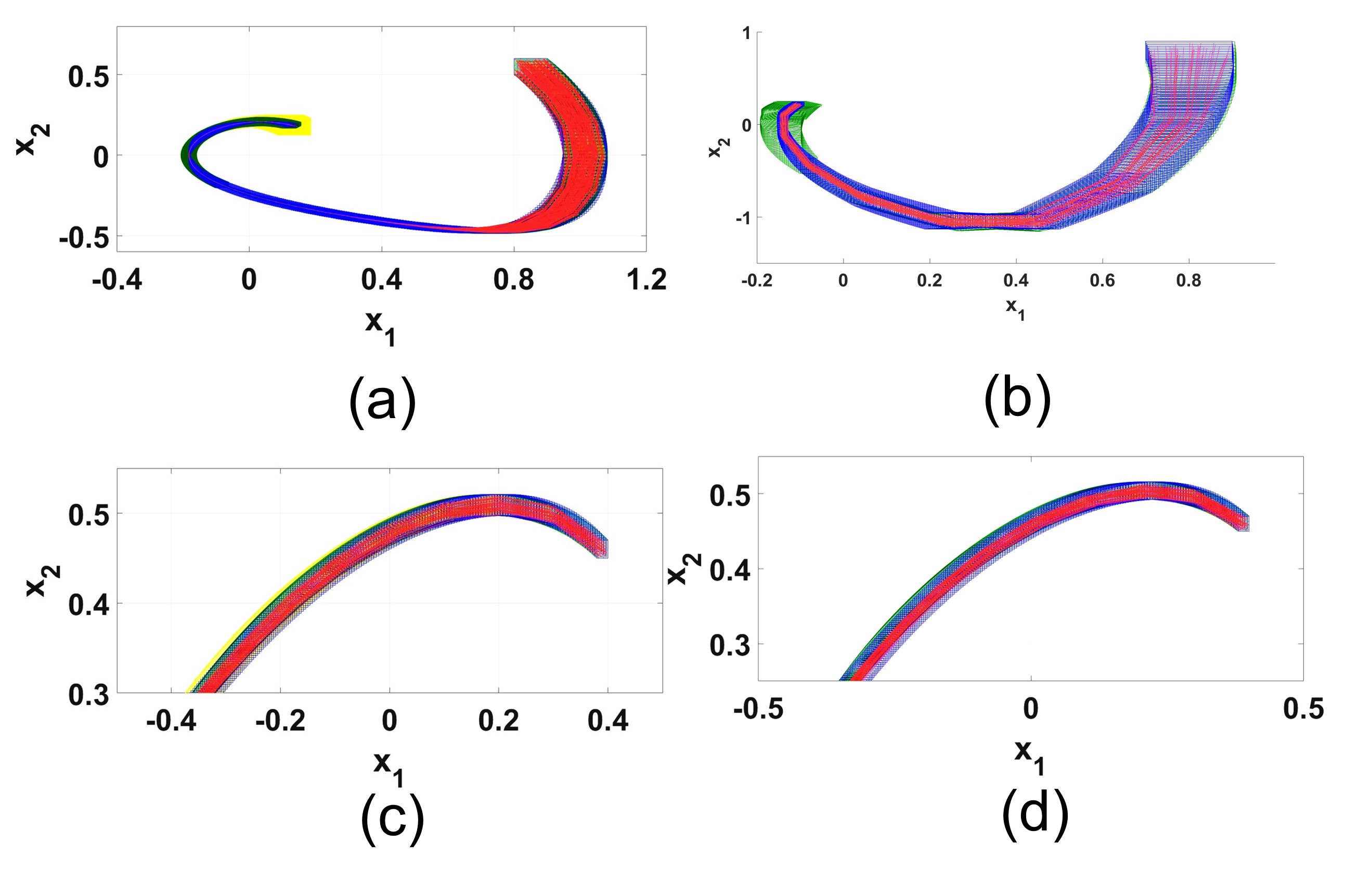}
\vspace{-4mm}
\caption{(a)B1 Sigmoid (b) B2 Sigmoid (c) B5 Sigmoid (d) B5 Tanh. Red trajectories are the simulated results. Blue reachable sets are generated from our approach. The yellow and green reachable sets are from the baseline methods.}
\label{reachable_set_estimation}
\vspace{-4mm}
\end{figure}

We compare the accuracy of the methods by calculating the area size of the reachable set at a predefined time point. The reachable set at the same time point estimated by different approaches is demonstrated in Figure \ref{final_reachable_set}. The comparison of area sizes of reachable sets is presented in Table \ref{reachable set}. Since directly obtaining an analytical ground-truth reachable set is difficult in NNCS, we use a grid-based exhaustive search with 1,000 points to approximate the ground-truth reachable set. The reachable states at the given time points are all located in the evaluated reachable set. We add a convex hull that covers all the reachable points. And we calculate and compare the area size with the reachable set.   The area size of DeepNNC is slightly larger than the exhaustive search, but consistently smaller than other baselines, revealing that DeepNNC can cover all the reachable examples with a tight estimation. 
Figure~\ref{reachable_set_estimation} shows the reachable sets with the evolving of the control timestep and the simulated trajectories (red). Reachable sets of DeepNNC (blue) can cover the trajectories with a small overestimation error.

\subsection{Impact of $\epsilon$ and $k_{max}$}
Two parameters $\epsilon$ and $_{max}k$ are involved in our approach. $\epsilon$ indicates the allowable error, and $k_{max}$ represents the maximum number of iterations in each dimension. The iteration stops either when the difference of upper bound (red square) and lower bound (blue square) is smaller than $\epsilon$ or the iteration number in one dimension reaches $k_{max}$. 
In Figure \ref{diff_k}, we fix $\epsilon=0.0005$ and change $k_{max}=3,5,10,10000$. The estimation becomes tighter with the larger iteration number. In Figure~\ref{diff_delta} we study the influence of $\epsilon$ by fixing the maximum iteration number $k_{max}=10^4$ and changing $\epsilon=0.5,0.05,0.005,0.0005$. The accuracy of the estimate increases with smaller $\epsilon$.


\begin{figure}[t]
\vspace{-2mm}
\centering
\includegraphics[width=0.45\textwidth]{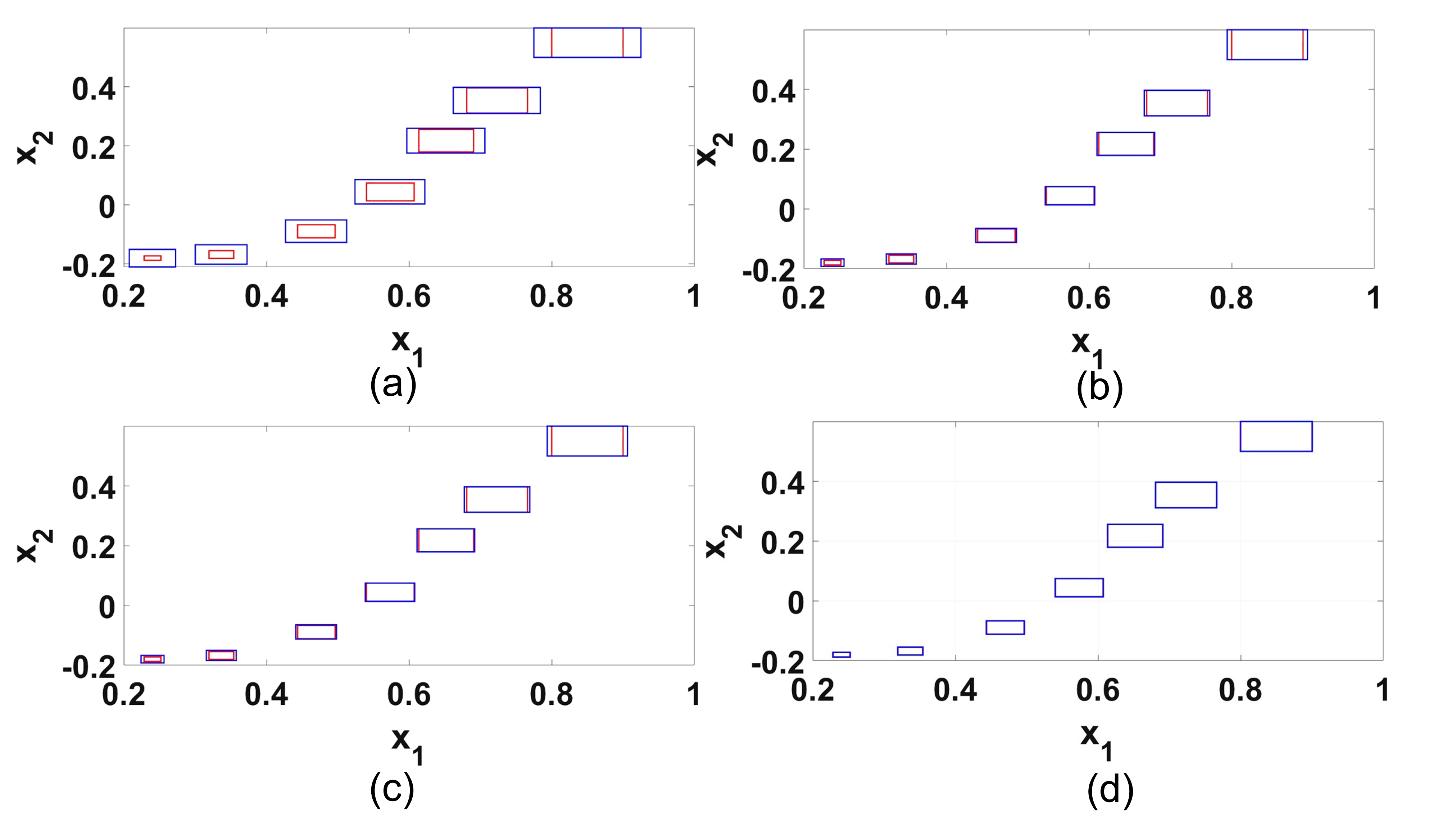}
\vspace{-4mm}
\caption{Reachable sets of B3 with ReLU at $t=0.01s$, 0.1s, 0.2s, 0.4s, 1s, 3s, 6s, with $\epsilon=0.0005$ and (a) $k_{max}=3$; (b) $k_{max}=5$; (c) $k_{max}=10$; (d) $k_{max}=10^4$}\label{diff_k}
\end{figure}


\begin{figure}[!ht]
\centering
\includegraphics[width=0.45\textwidth]{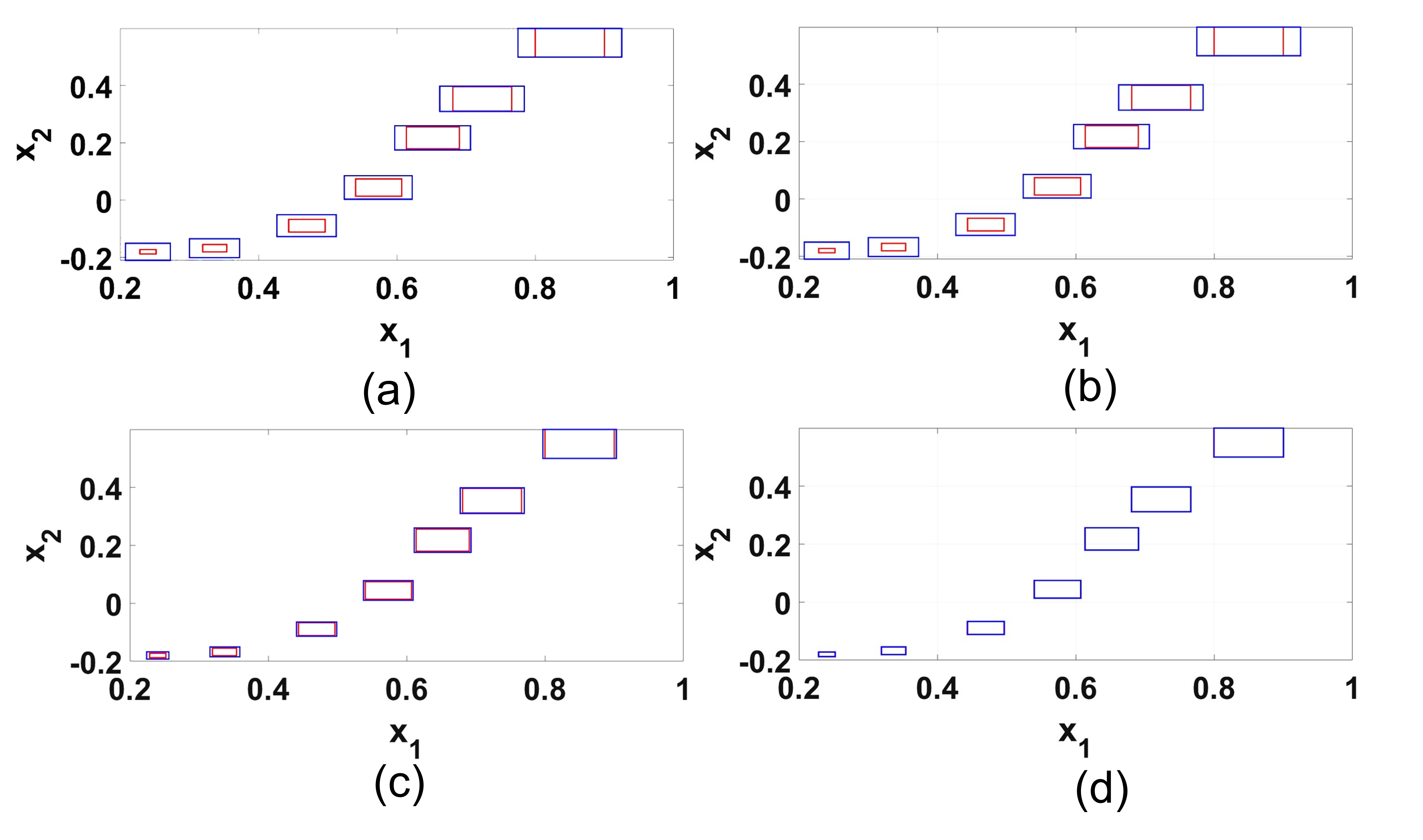}
\vspace{-4mm}
\caption{Reachable sets of B3 with ReLU at $t=0.01s$, 0.1s, 0.2s, 0.4s, 1s, 3s, 6s, with $k_{max}=10^4$ and (a) $\epsilon=0.5$; (b) $\epsilon=0.05$; (c) $\epsilon=0.005$; (d) $\epsilon=0.0005$}\label{diff_delta}
\end{figure}

\begin{figure}[!h]
\centering
\includegraphics[width=0.23\textwidth]{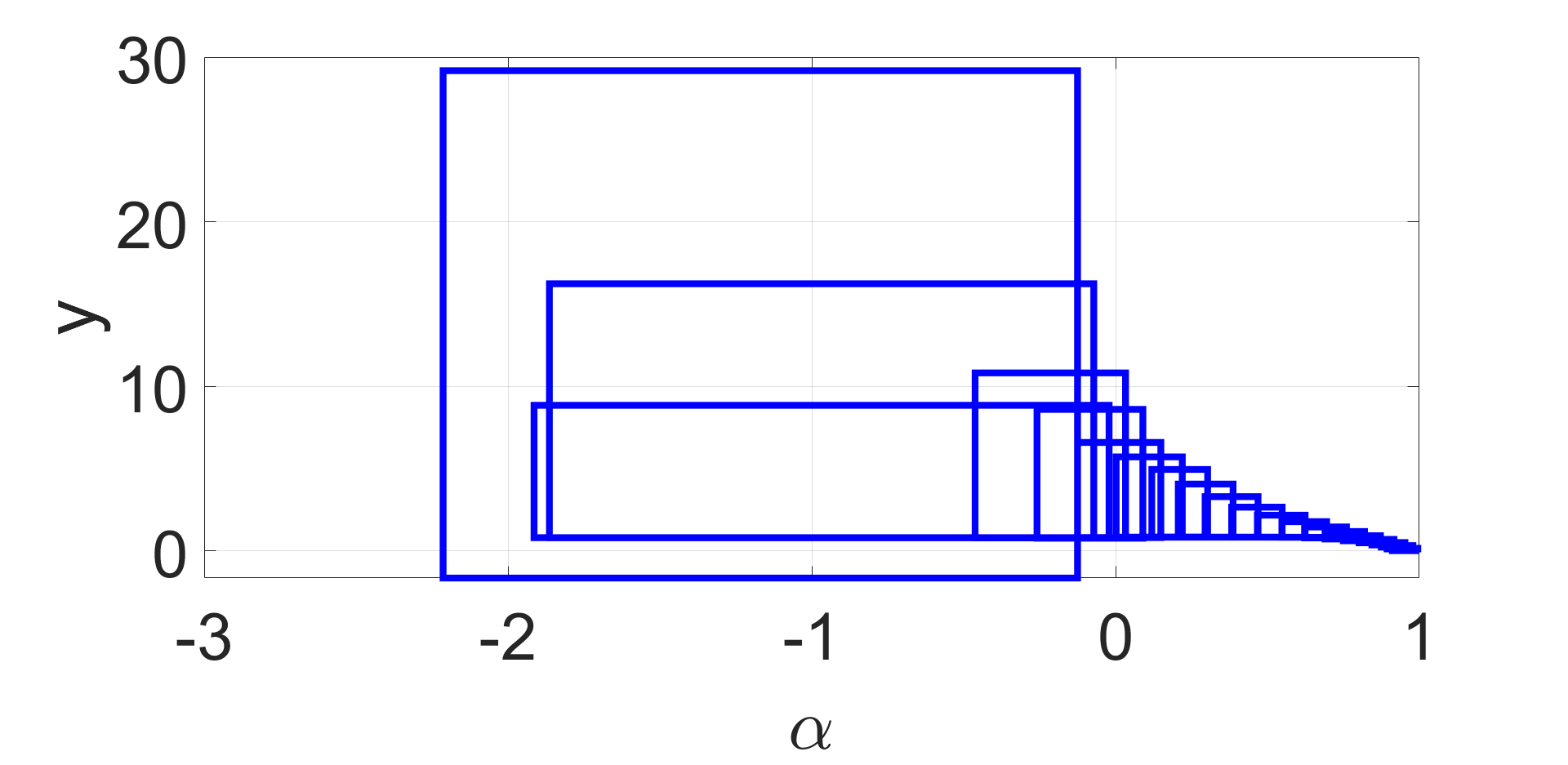}
\includegraphics[width=0.23\textwidth]{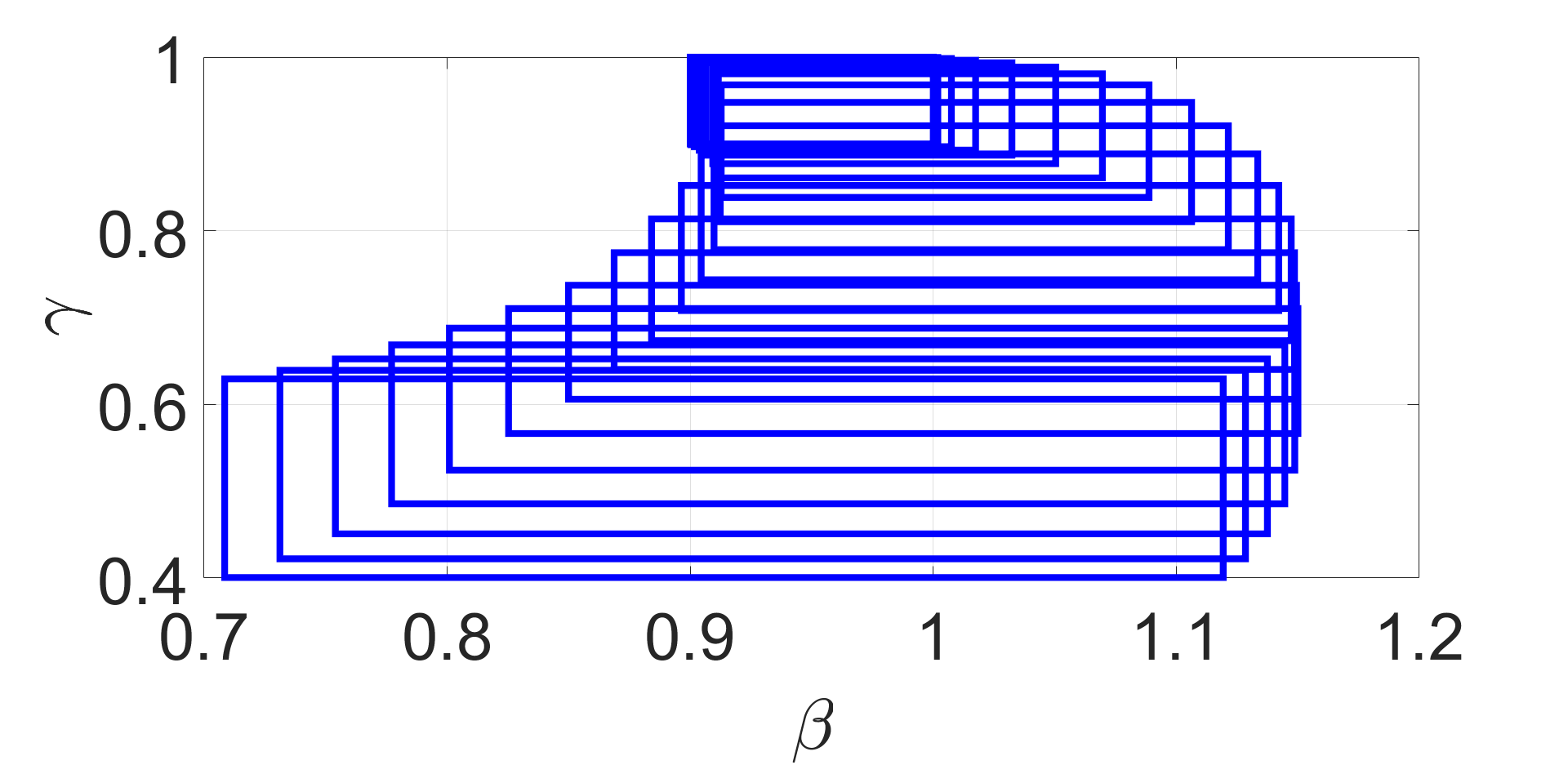}
\vspace{-3mm}
\caption{Reachable sets of $\alpha-y$ and $\beta$ and $\gamma$ at $t=[0s, 2s]$ with an initial subset $ [u,v,w,\alpha, \beta, \gamma]=[0.9,1]^6$}
\label{airplane_result}
\vspace{-4mm}
\end{figure}


\subsection{Case Study: Flying Airplane}
We analyse a complex control system of a flying airplane, which is a benchmark in ARCH-COMP21 \cite{ARCH21:ARCH_COMP21_Category_Report_Artificial}. The system contains 12 states $[x,y,z,u,v,w,\alpha,\beta,\gamma,r,p,q]$ and the technical details are in {\bf Appendix-F}. Initial states are $ x=y=z=r=p=q=0,
    [u,v,w,\alpha, \beta, \gamma]=[0,1]^6.$ 
The goal set is $y\in [-0.5,0,5]$ and $[\alpha, \beta, \gamma]=[-1,1]^3$ for  $t<2s$.
The controller here takes an input of 12 dimensions and generates a six-dimensional control input $F_x, F_y, F_z, M_x, M_y, M_z$. 
For simplicity, we choose a small subset of the initial input $[u,v,w,\alpha, \beta, \gamma]=[0.9,1]^6$, the estimated reachable sets of $\alpha$ and $y$ in the time range $[0s, 2s]$ are presented in figure \ref{airplane_result} (a). Both $\alpha$ and $y$ are above the safe range. Figure \ref{airplane_result} shows the reachable set of $\alpha$ and $\gamma$, which is also above the safe region. The airplane NNCS is unsafe.


\section{Conclusion}

We develop an NNCS verification tool, DeepNNC, that is applicable to a wide range of neural network controllers, as long as the whole system is Lipschitz-continuous. We treat NNCS verification as a series of optimisation problems, and the estimated reachable set bears a tight bound even in a long control time horizon. The efficiency of DeepNNC is mainly influenced by two factors, the Lipschitz constant estimation and the dimension of the input. For the first, we have adopted dynamic local Lipschitzian optimisation to improve efficiency. Regarding the high-dimensional reachability problem, a possible solution is to transform the high-dimension problem into a one-dimensional problem using a space-filling curve~\cite{lera2015deterministic}. This idea can be explored in future work. 

\bigskip
\section{Acknowledgements}
This work is supported by Partnership Resource Fund of ORCA Hub via the UK EPSRC under project [EP/R026173/1]. 

\balance
\bibliography{aaai23}

\end{document}